\documentclass[review,a4paper,fleqn,hyperref={colorlinks=true,bookmarksopen=true,citepcolor=red,urlcolor=red}]{cas-sc}
\usepackage[switch]{lineno}
\usepackage{setspace}
\usepackage{soul, color, xcolor}

\soulregister\citep7
\soulregister\ref7
\soulregister\eqref7
\soulregister\citepp7

\usepackage{amsmath,amssymb,amsthm}
\usepackage{amsmath}

\usepackage{caption}
\usepackage[authoryear]{natbib}

\usepackage{bm}
\usepackage{url}
\usepackage{graphicx}
\usepackage{multicol}  %
\usepackage{color, soul, framed}
\usepackage{algorithm, algorithmicx, algpseudocode}
\usepackage{listings}%
\usepackage{alltt}%

\sethlcolor{yellow}
\setstcolor{green}
\setulcolor{red} 

\def\tsc#1{\csdef{#1}{\textsc{\lowercase{#1}}\xspace}}
\tsc{WGM}
\tsc{QE}
\usepackage{graphicx}
\usepackage{subcaption}
\usepackage{amsthm}
\newtheorem{theorem}{Theorem}
\usepackage{float}
\usepackage{afterpage}

\usepackage[switch]{lineno}

\begin{document}

\let\WriteBookmarks\relax
\def\floatpagepagefraction{1}
\def\textpagefraction{.001}

\shorttitle{Adaptive Line-Of-Sight path following for underactuated unmanned surface vehicle}    

\shortauthors{Jie Qi et~al.}  

\title [mode = title]{Adaptive Line-Of-Sight guidance law based on vector fields path following for underactuated unmanned surface vehicle}





\author[1,2]{Jie Qi}[style=chinese]
\credit{Writing, Original draft preparation, Conceptualization of this study, Methodology}

\author[1,2]{Ronghua Wang}[style=chinese]
\credit{Writing, Original draft preparation, Conceptualization of this study, Methodology}

\author[1,2]{Nailong Wu}[style=chinese,orcid=0000-0003-4869-6361]
\credit{Supervision, Conceptualization of this study, Revision}
\cormark[1]
\ead{nathan_wu@dhu.edu.cn}

\author[1,2]{Yuxin Fan}[style=chinese]
\credit{Formal analysis, Visualization}

\author[1,2]{Jigang Wang}[style=chinese]
\credit{Formal analysis, Visualization}

\address[1]{College of Information Science and Technology, Donghua University, Shanghai 201620, China}

\address[2]{Engineering Research Center of Digitized Textile $\&$ Apparel Technology, Ministry of Education, Donghua University, Shanghai 201620, China}


\cortext[1]{Corresponding author} 


\begin{abstract}
The focus of this paper is to develop a methodology that enables an unmanned surface vehicle (USV) to efficiently track a planned path. The introduction of a vector field-based adaptive line-of-sight guidance law (VFALOS) for accurate trajectory tracking and minimizing the overshoot response time during USV tracking of curved paths improves the overall line-of-sight (LOS) guidance method. These improvements contribute to faster convergence to the desired path, reduce oscillations, and can mitigate the effects of persistent external disturbances. It is shown that the proposed guidance law exhibits $\kappa$-exponential stability when converging to the desired path consisting of straight and curved lines. The results in the paper show that the proposed method effectively improves the accuracy of the USV tracking the desired path while ensuring the safety of the USV work.
\end{abstract}

\begin{keywords}
  Adaptive line-
Of-sight guidance law\sep Vector field\sep Unmanned surface vehicle
\end{keywords}

\maketitle


\section{Introduction}
In recent years, unmanned surface vehicles have gained widespread attention for their ability to perform maritime missions more efficiently, cost-effectively, and with greater safety \citep{barrera2021trends, liu2020adaptive}. To effectively accomplish the given task, an unmanned surface vehicle needs to possess autonomous or semi-autonomous navigation capabilities. In this scenario, a guidance system plays a crucial role in guiding the vehicle along a predetermined path, ensuring convergence and stability\citep{Breivik2009, pengusv2021}. However, the tracking performance of USV may be hampered by unknown maritime drift forces or interference caused by ocean currents and time-varying waves \citep{Borhaug2008}. To address this challenge, alternative approaches have been proposed. \citet{Nelson2007} using vector fields around the path to steer the vehicle, or \citet{Caharija2016} using integral LOS steering to mitigate the effects of unknown drift forces. Inspired by the above discussion, this paper considers the static path problem of USVs and investigates the accuracy and stability of USVs tracking the target path. The tracking tasks and tests considered in this paper are mainly realized in a two-dimensional horizontal plane. The main contributions of this paper are as follows:
\begin{itemize}
    \item 
    Based on the idea of vector field, an adaptive line-of-Sight guidance law for unmanned boats is proposed to improve the steering performance of the unmanned boat and provide damping characteristics across the track for error overshooting.
    \item 
    A stability analysis is provided to show that the proposed guidance law is $\kappa$-exponential stable when converging to an ideal path consisting of straight and curved lines.
    \item 
    The effectiveness of the proposed guidance law and its robustness against unpredictable external disturbances are verified through simulations using the robotic operating system ROS as well as on-board lake experiments in a path-following task.
\end{itemize}

\section{USV model}\label{sec2}

In this paper, the actual USV model is employed in the lake experiments of the CCPP algorithm. For this purpose, the USV is modeled, and the controller
of the actual USV model is embedded in the simulation
environment. Fig. \ref{liucheng} illustrates the flowchart of the control strategy for the LOS guidance law.

\begin{figure*}
\centering
\includegraphics[width=16cm]{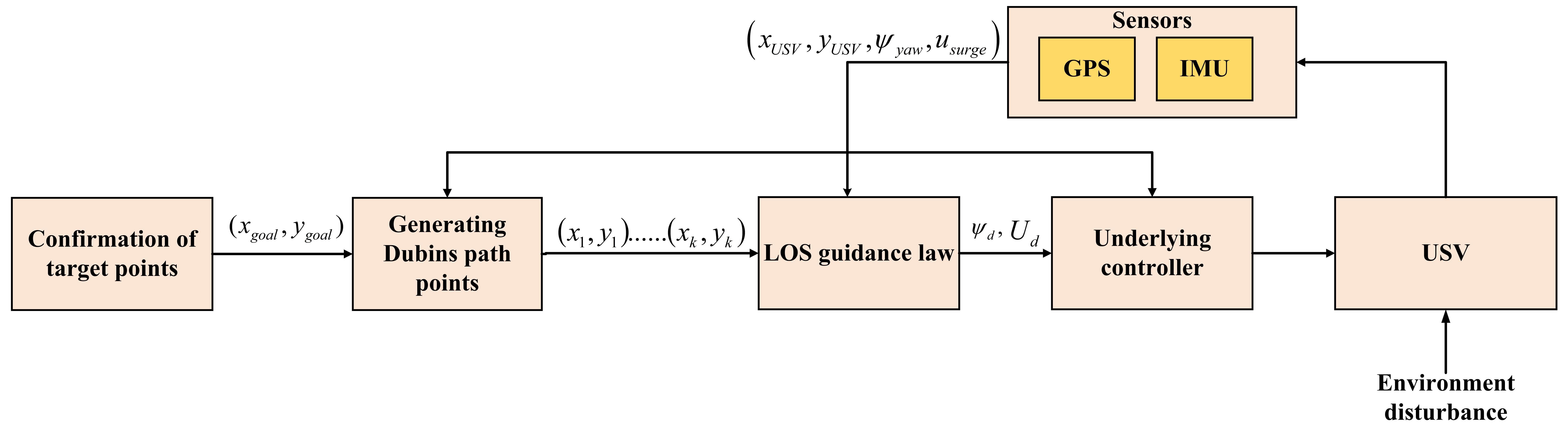}
\caption{Flowchart of the control strategy for the LOS guidance law.}
\label{liucheng}
\end{figure*}

The North-East-Down (NED) coordinate system and the body coordinate system are used to describe the motion, position, and attitude of the USV, respectively \citep{mccue2016handbook}. As shown in Fig. \ref{USVmode}, $({{x}_{1}},{{y}_{1}})$ represents the position of the USV under the NED coordinate system, and $({{x}_{2}},{{y}_{2}})$ denotes the status value in the body coordinate system of the USV. 
$\mathbf{\eta }={{[x,y,\psi ]}^{T}}$ is the state vector, and $\mathbf{v}={{[u,v,r]}^{T}}$ denotes  the velocity vector of the USV. The kinematic model of the USV is given by:
\begin{eqnarray}
&{{\dot{x}}_{1}}=u\cos (\psi )-v\sin (\psi )\\
&{{\dot{y}}_{1}}=u\sin (\psi )+v\cos (\psi )\\
&{\dot{x}}_2=u=({{u}_{1}}+{{u}_{2}})/2\\
&{\dot{y}}_2=v\\
&r=({{u}_{1}}-{{u}_{2}})/a\\
&U=\sqrt{{{u}^{2}}+{{v}^{2}}}
\end{eqnarray}
where $\dot{(\cdot)}$ denotes the differentiation with respect to time. $a$ is the distance between the two propellers. $u$ and $v$ represent the velocity values of the USV in the body coordinate system along the $x_2$ and $y_2$ axes, respectively. $\psi$ is the yaw angle, and $r$ means the yaw rate. $U$ means the speed value of the USV. $u_1$ and $u_2$ represent the velocity values of the left propeller and the right propeller, respectively.

\begin{figure}
    \centering  \includegraphics[width=5cm]{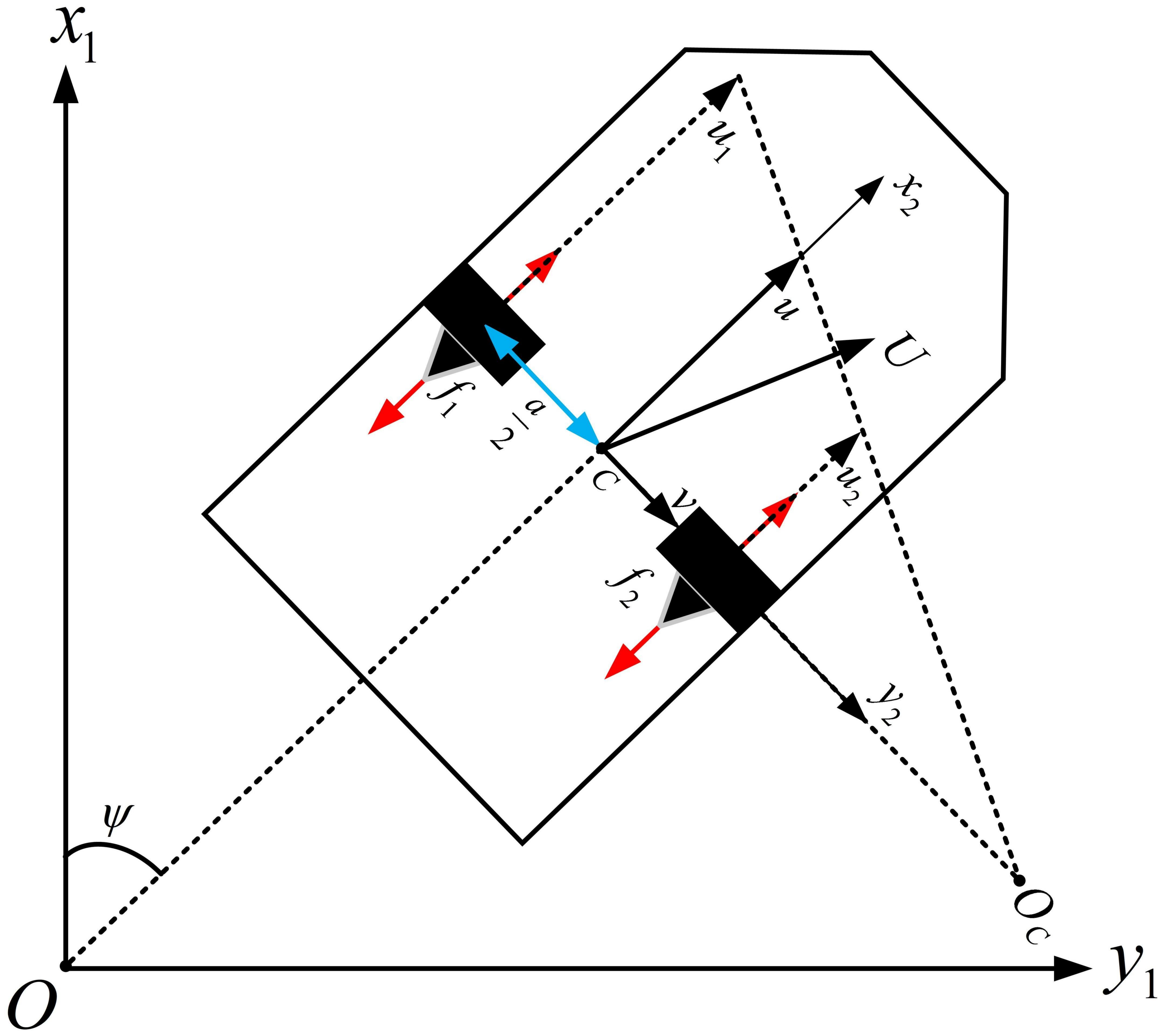}
    \caption{3-DOF vessel centered at $C$ in the North-East-Down reference frame.}
    \label{USVmode}
\end{figure}

Assuming that the USV operates as a rigid body, its dynamical equations can be expressed as follows:
\begin{eqnarray}
&M\mathbf{\dot{v}}+\mathbf{C(v)v}+\mathbf{D(v)v}=\mathbf{\tau }\\
&M={{M}_{RB}}+{{M}_{A}}
\end{eqnarray}
where ${M}_{RB}$ represents the mass of the USV, and ${M}_{A}$  denotes the added mass of the USV due to the resistance of the water during accelerated navigation. $\mathbf{C(v)}$ means the drag matrix generated by the USV when it is rotating in water.  $\mathbf{D(v)}$ signifies the damping matrix, the friction damping force generated by the USV in motion, which is mainly related to the surface area of the unmanned ship.  $\tau ={{[{{\tau }_{u}},{{\tau }_{r}}]}^{T}}$ is the thrust matrix of the propeller with the following expression:
\begin{eqnarray}
&{{\tau }_{u}}={{f}_{1}}+{{f}_{2}}
\label{tuili}\\
&{{\tau }_{r}}=\frac{a}{2}({{f}_{1}}-{{f}_{2}})\\
&\mathbf{\tau }=\mathbf{Bu}=\left[ \begin{matrix}
   1 & 1  \\
   a/2 & -a/2  \\
\end{matrix} \right]\left[ \begin{matrix}
   {{f}_{1}}  \\
   {{f}_{2}}  \\
\end{matrix} \right]
\end{eqnarray}
where $\mathbf{B}$ is the propulsion configuration of the thrusters and $\mathbf{u}={{[{{f}_{1}},{{f}_{2}}]}^{T}}$ denotes the thrust vector of the two propellers. Note that the $\mathbf{B}$ is defined in the controller firmware.

The USV's main controller, the Pixhawk, is equipped with the Ardupilot Boat Firmware, which integrates the USV's motion control system. The output results calculated by the PID behavior controller are distributed according to the thruster mechanics influence coefficient matrix to control the motor speed.

\section{Vector field-based adaptive LOS guidance law}

The actual tracking of path planning depends on the performance of the guidance control algorithm. Therefore, it is necessary to design an appropriate guidance law to ensure the stability of USV \citep{liu2016predictor}.
Fig. \ref{fig1} describes the geometry of the path LOS guidance law and some variables involved \citep{wan2020improved, qiu2020predictor}. Some of the most critical variables include:
\begin{eqnarray}
&{{y}_{e}}=-(x-{{x}_{p}})\sin ({{\gamma }_{p}})+(y-{{y}_{p}})\cos ({{\gamma }_{p}})\\
&{{\psi }_{d}}={{\gamma }_{p}}-\arctan \left( \frac{{{y}_{e}}}{\Delta } \right) - \beta 
\end{eqnarray}
where ${{y}_{e}}$ represents the cross-track error, ${{\psi }_{d}}$ is the desired heading angle, and $\beta $ signifies the sideslip angle, and $\Delta $ represents the look-ahead distance, which determines how aggressively the USV steers. Besides, ${{\gamma }_{p}}$ names the path tangential angle.

When the sideslip angle changes rapidly due to time-varying environmental disturbances,  the adaptive LOS guidance law can directly compensate for the unknown sideslip angle to ensure good tracking capacity \citep{yuan2022study, du2023improved}. The adaptive LOS guidance law proposed by \citet{fossen2023adaptive} is given by:
\begin{eqnarray}
&{{\psi }_{d}}={{\gamma }_{p}}-\hat{\beta }-\arctan \left( \frac{{{y}_{e}}}{\Delta } \right)
\label{VFALOS1}\\
&\dot{\hat{\beta }}=\gamma \frac{\Delta {{y}_{e}}}{\sqrt{{{\Delta }^{2}}+y_{e}^{2}}}
\label{VFALOS2}
\end{eqnarray}
where $\gamma$ represents the adaptation gain, and $\hat{\beta }$ is the parameter estimate of ${\beta }$. The parameter estimation error for the sideslip angle is $\tilde{\beta }=\beta -\hat{\beta }$. When integral control is used instead of parameter adaptation, $\hat{\beta }\equiv 0$ and therefore $\tilde{\beta }=\beta $. Assume that $\beta$ is small and constant throughout the path tracing,  
then we have $\tilde{\beta }=\beta \approx 0$, $\dot{\beta }=0$ and $\dot{\tilde{\beta }}\approx -\dot{\hat{\beta }}$.


\begin{figure}
    \centering
    \begin{subfigure}[b]{0.5\linewidth}
        \centering
        \includegraphics[width=7cm]{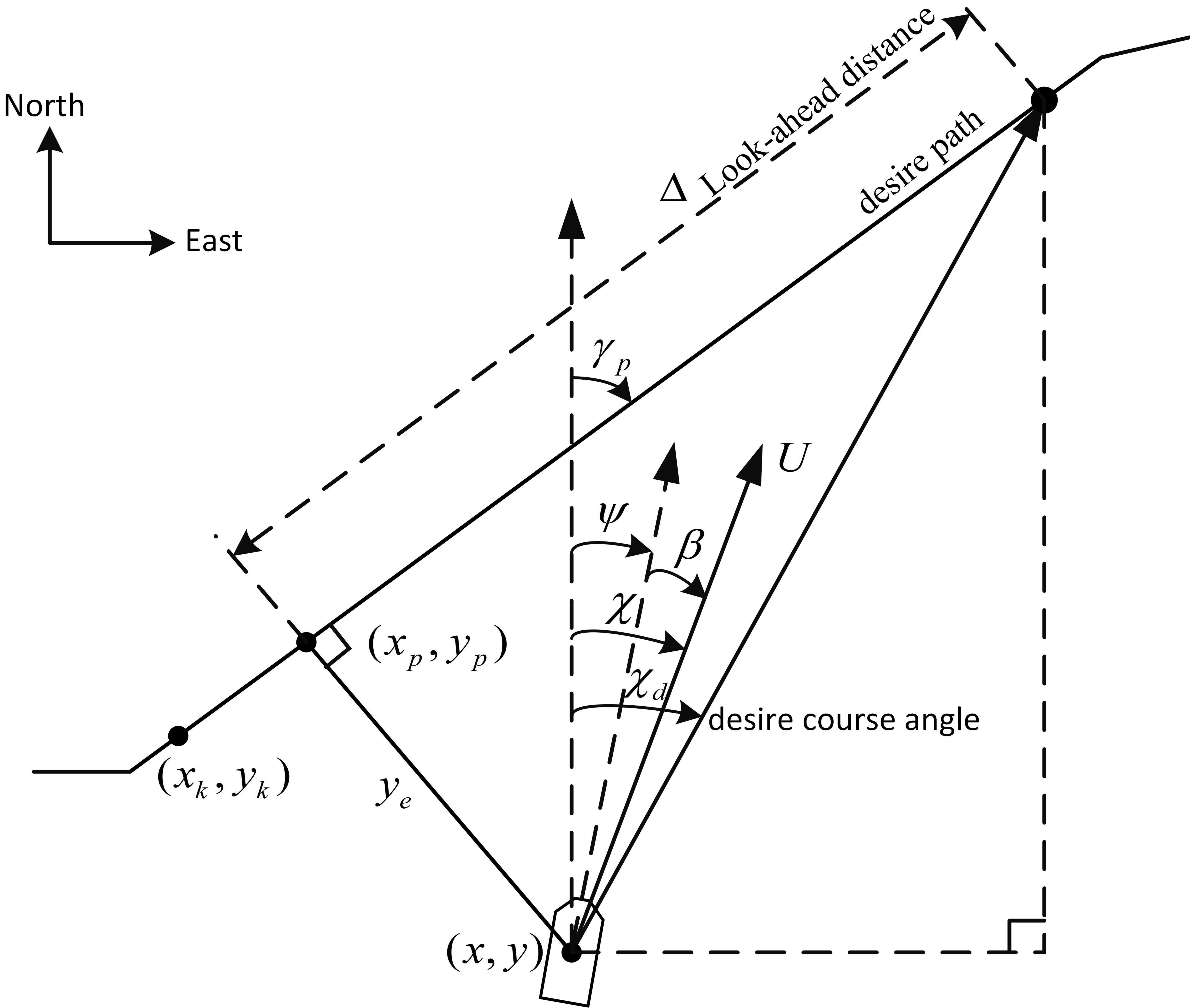}
        \caption{Geometry of the LOS guidance law.}
        \label{fig1}
    \end{subfigure}%
    \begin{subfigure}[b]{0.5\linewidth}
        \centering
        \includegraphics[width=5cm]{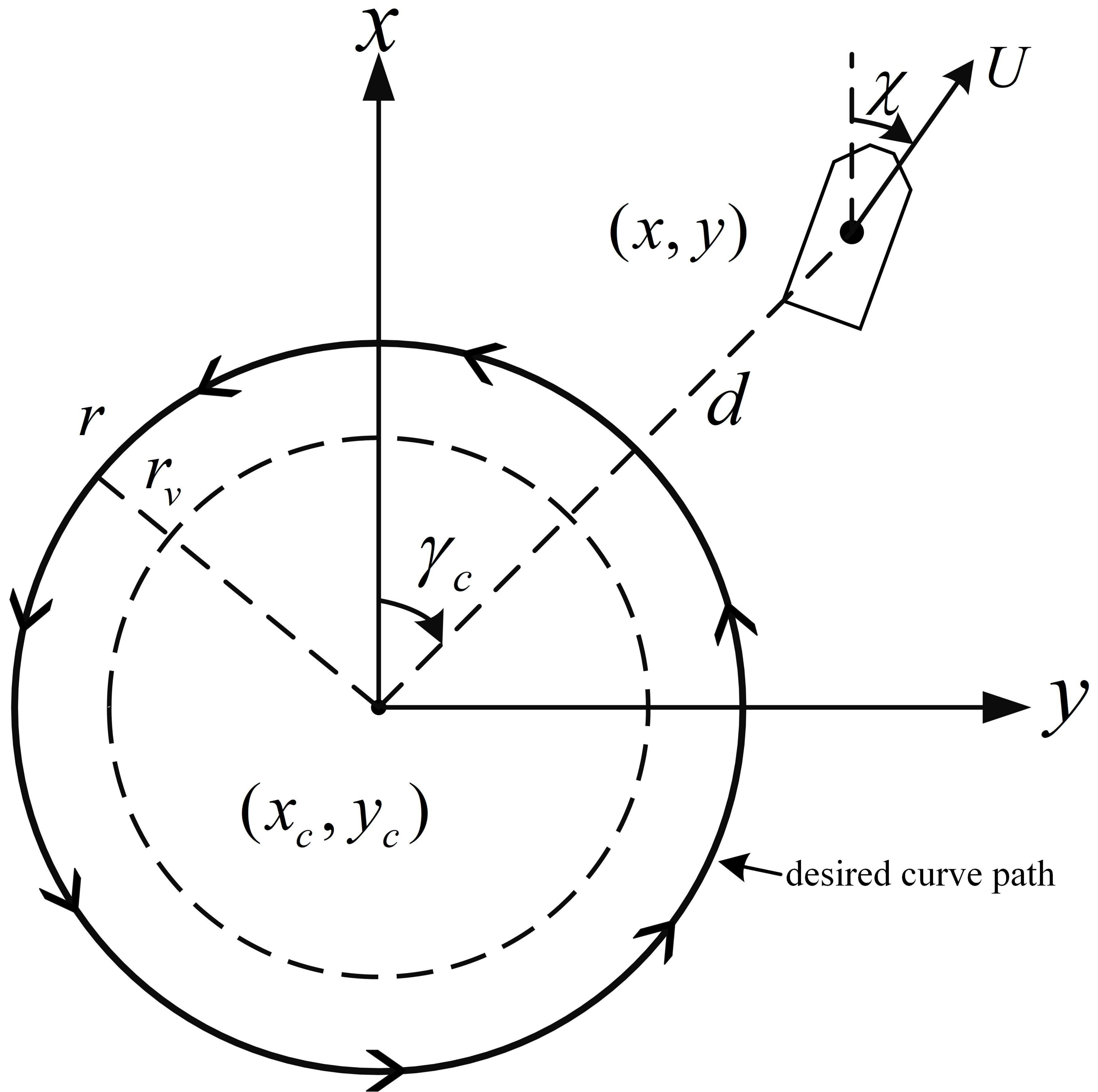}
        \caption{Vector field geometry for curved route tracking.}
        \label{curve-LOS}
    \end{subfigure}
    \caption{Schematic diagram of the guidance law.}
    \label{Schematic diagram of the guidance law}
\end{figure}
USV is prone to overshooting when tracking a curved path in actual water. It is easier to bring more oscillations when the USV tracks the desired path after a turn. \citet{nelson2007vector} proposed the vector field path tracking method, which has good performance in tracking turn paths. Therefore, this paper proposes to introduce a vector field in the adaptive LOS guidance law to reduce the overshoots and oscillations of USV when tracking curved paths. As shown in Fig. \ref{curve-LOS}, $y_e$ can be denoted as ${{y}_{e}}\approx d-r$ and ${\gamma }_{p}$ can be denoted as ${{\gamma }_{p}}\approx {{\gamma }_{c}}-\pi /2$. ${\gamma }_{c}$ is the azimuth of the USV relative to the center of the desired curve path circle. The vector field-based adaptive LOS guidance law  proposed in this paper is used for USV tracking curve paths and is defined as follows:
\begin{eqnarray}
&{{\psi }_{d}}={{\gamma }_{c}}-\frac{\pi }{2}-\hat{\beta }-\arctan (\frac{d-{{r}_{v}}}{\Delta })\
\label{VFALOS3}\\
&\dot{\hat{\beta }}=\gamma \frac{\Delta (d-{{r}_{v}})}{\sqrt{{{\Delta }^{2}}+{{(d-{{r}_{v}})}^{2}}}}
\label{VFALOS4}
\end{eqnarray}
where $d$ is the distance between the current position of the USV and the center of the curved path. ${r}_{v}$  represents the vector field radius, which is related to the path radius $r$ as follows:
\begin{eqnarray}
{{r}_{v}}=\arctan (kr)\frac{2 }{\pi }{{r}_{\min }}+(r-{{r}_{\min }})
\end{eqnarray}
where $k$ and ${r}_{min}$ are the designed parameters. When the radius r of the path is not constant, it can be calculated according to the following equation:
\begin{eqnarray}
r=\frac{\sqrt{{{({{x}_{p+1}}-{{x}_{p}})}^{2}}+{{({{y}_{p+1}}-{{y}_{p}})}^{2}}}}{2\sin \left( ({{{\gamma }_{p+1}}-{{\gamma }_{p}})}/{2}\; \right)}
\label{r}
\end{eqnarray}
where $(x_p,y_p)$ is the point on the desired path closest to the USV. $({{x}_{p+1}},{{y}_{p+1}})$ represents the coordinates of the expected next point on the path, and its tangent angle for the expected path is ${\gamma }_{p+1}$.

\begin{theorem}[Globally $\kappa$-Exponentially Stable VFALOS Guidance Law]\label{th}

In the face of constant or time-varying environmental disturbances, the nominal system's point $d=r$ is globally $\kappa $-exponentially stable if the intended heading angle is represented by Eq. \eqref{VFALOS3} and the time derivative of the parameter estimation is provided by Eq. \eqref{VFALOS4}.
\end{theorem}
\begin{proof}
According to \citet{wang2023vector}, the time derivative of ${y}_{e}$ is the nominal system  $\Sigma$, which is given by:
\begin{eqnarray}
{{\dot{y}}_{e}}&=&U\sin (\psi -{{\gamma }_{p}}+\beta )\nonumber\\
&=&U\sin (\chi -{{\gamma }_{p}})
\label{mingyi}
\end{eqnarray} 
Eqs. \eqref{VFALOS3} and \eqref{VFALOS4} can be rewritten as:
\begin{eqnarray}
{{\psi }_{d}}={{\gamma }_{p}}-\hat{\beta }-\arctan \left( \frac{{{y}_{e}}+r-{{r}_{v}}}{\Delta } \right)
\label{VFALOS5}
\end{eqnarray}
\begin{eqnarray}
\dot{\hat{\beta }}=\gamma \frac{\Delta ({{y}_{e}}+r-{{r}_{v}})}{\sqrt{{{\Delta }^{2}}+{{({{y}_{e}}+r-{{r}_{v}})}^{2}}}}
\label{VFALOS6}
\end{eqnarray}
Substituting Eq. \eqref{VFALOS5} into Eq. \eqref{mingyi}:
\begin{eqnarray}
\begin{aligned}
{{{\dot{y}}}_{e}}=&U\sin (\psi -{{\gamma }_{p}}+\beta )\\
=&U\sin \left( \tilde{\beta }-\arctan \left( \frac{{{y}_{e}}+r-{{r}_{v}}}{\Delta } \right) \right) \\ 
=&U\sin \tilde{\beta }\cos \left( \arctan \left( \frac{{{y}_{e}}+r-{{r}_{v}}}{\Delta } \right) \right) \\ 
&-U\cos \tilde{\beta }\sin \left( \arctan \left( \frac{{{y}_{e}}+r-{{r}_{v}}}{\Delta } \right) \right) \\ 
\approx& U\tilde{\beta }\cos \left( \arctan \left( \frac{{{y}_{e}}+r-{{r}_{v}}}{\Delta } \right) \right) \\ 
&-U\sin \left( \arctan \left( \frac{{{y}_{e}}+r-{{r}_{v}}}{\Delta } \right) \right)
\end{aligned}
\label{tuidao}
\end{eqnarray}
$\sin (\arctan (x))$ and $\cos (\arctan (x))$ can be rewritten as the following equations:
\begin{eqnarray}
\begin{aligned}
& \sin (\arctan (x))=x/\sqrt{{{x}^{2}}+1} \\ 
& \cos (\arctan (x))=1/\sqrt{{{x}^{2}}+1}  \\
\end{aligned}
\label{sanjiaozhuanhuan}
\end{eqnarray}
Substituting Eq. \eqref{sanjiaozhuanhuan} into Eq. \eqref{tuidao}:
\begin{eqnarray}
{{\dot{y}}_{e}}=\frac{\tilde{\beta }\Delta U}{\sqrt{{{\Delta }^{2}}+{{({{y}_{e}}+r-{{r}_{v}})}^{2}}}}-\frac{({{y}_{e}}+r-{{r}_{v}})U}{\sqrt{{{\Delta }^{2}}+{{({{y}_{e}}+r-{{r}_{v}})}^{2}}}}
\label{dotye}
\end{eqnarray}
Setting the Lyapunov function candidate(LFC):
\begin{eqnarray}
V&=&\frac{1}{2}{{(d-{{r}_{v}})}^{2}}+\frac{U}{2\gamma }{{\tilde{\beta }}^{2}}\nonumber\\
&\approx& \frac{1}{2}{{({{y}_{e}}+r-{{r}_{v}})}^{2}}+\frac{U}{2\gamma }{{\tilde{\beta }}^{2}}
\label{2020}
\end{eqnarray}
Eq. \eqref{dotye} is substituted into the time derivative of the LFC to yield:
\begin{eqnarray}
&\dot{V}\approx ({{y}_{e}}+r-{{r}_{v}}){{\dot{y}}_{e}}+\frac{U}{\gamma }\tilde{\beta }\dot{\tilde{\beta }}\nonumber\\
&=\frac{({{y}_{e}}+r-{{r}_{v}})\tilde{\beta }\Delta U}{\sqrt{{{\Delta }^{2}}+{{({{y}_{e}}+r-{{r}_{v}})}^{2}}}}-\frac{{{({{y}_{e}}+r-{{r}_{v}})}^{2}}U}{\sqrt{{{\Delta }^{2}}+{{({{y}_{e}}+r-{{r}_{v}})}^{2}}}}+\frac{U}{\gamma }\tilde{\beta }\dot{\tilde{\beta }}
\label{dotV}
\end{eqnarray}
Since $\dot{\tilde{\beta }}\approx -\dot{\hat{\beta }}$, substituting Eq. \eqref{VFALOS6} into Eq. \eqref{dotV}:
\begin{eqnarray}
&\dot{V}\approx -\frac{U{{({{y}_{e}}+r-{{r}_{v}})}^{2}}}{\sqrt{{{\Delta }^{2}}+{{({{y}_{e}}+r-{{r}_{v}})}^{2}}}}\nonumber\\
&=-\frac{U{{(d-{{r}_{v}})}^{2}}}{\sqrt{{{\Delta }^{2}}+{{(d-{{r}_{v}})}^{2}}}}
\end{eqnarray}
where $U>0$, so it is negatively definite. Therefore, $d=r$ is the uniform global asymptotic stability (UGAS) equilibrium for system \eqref{mingyi}. For $D=\{d\in R,\left| d-r \right|\le \xi \}$, $\xi >0$, there is when $0<M<U/\sqrt{{{\Delta }^{2}}+{{(\xi +r-{{r}_{v}})}^{2}}}$:
\begin{eqnarray}
\dot{V}\le -U\frac{{{(d-{{r}_{v}})}^{2}}}{\sqrt{{{\Delta }^{2}}+{{(\xi +r-{{r}_{v}})}^{2}}}}\le -M{{(d-{{r}_{v}})}^{2}}
\end{eqnarray}
This means that $d=r$ is uniform local exponential stability (ULES) equilibrium for system \eqref{mingyi}. According to \citet{lefeber2000tracking}, the combination of UGAS and ULES implies global $\kappa $-exponential
stability.
\end{proof}

The desired velocity of USV for tracking straight and curved paths is different to ensure tracking accuracy. The velocity distribution during USV's path tracking is given by the following equation:
\begin{eqnarray}
{{U}_{d}}=\max \left( {{U}_{\max }}(1-\frac{|{{y}_{e}}|}{{{y}_{\max }}}-\frac{|\tilde{\chi }|}{{{\chi }_{\max }}})+{{U}_{\min }}, {{U}_{\min }} \right)
\end{eqnarray}
where ${U}_{d}$ means the desired velocity, and the course angle error $\tilde{\chi}$ is: $\tilde{\chi }=\chi -{{\chi }_{d}}$. ${y}_{\max }$ is the maximum allowable cross-track error, and ${\chi }_{\max }$ represents the maximum allowable course angle error, both of which are design parameters. ${U}_{max}$ and ${U}_{min}$ denote the minimum and maximum values of the speed setting, respectively. 

\begin{figure}
    \centering  \includegraphics[width=9cm]{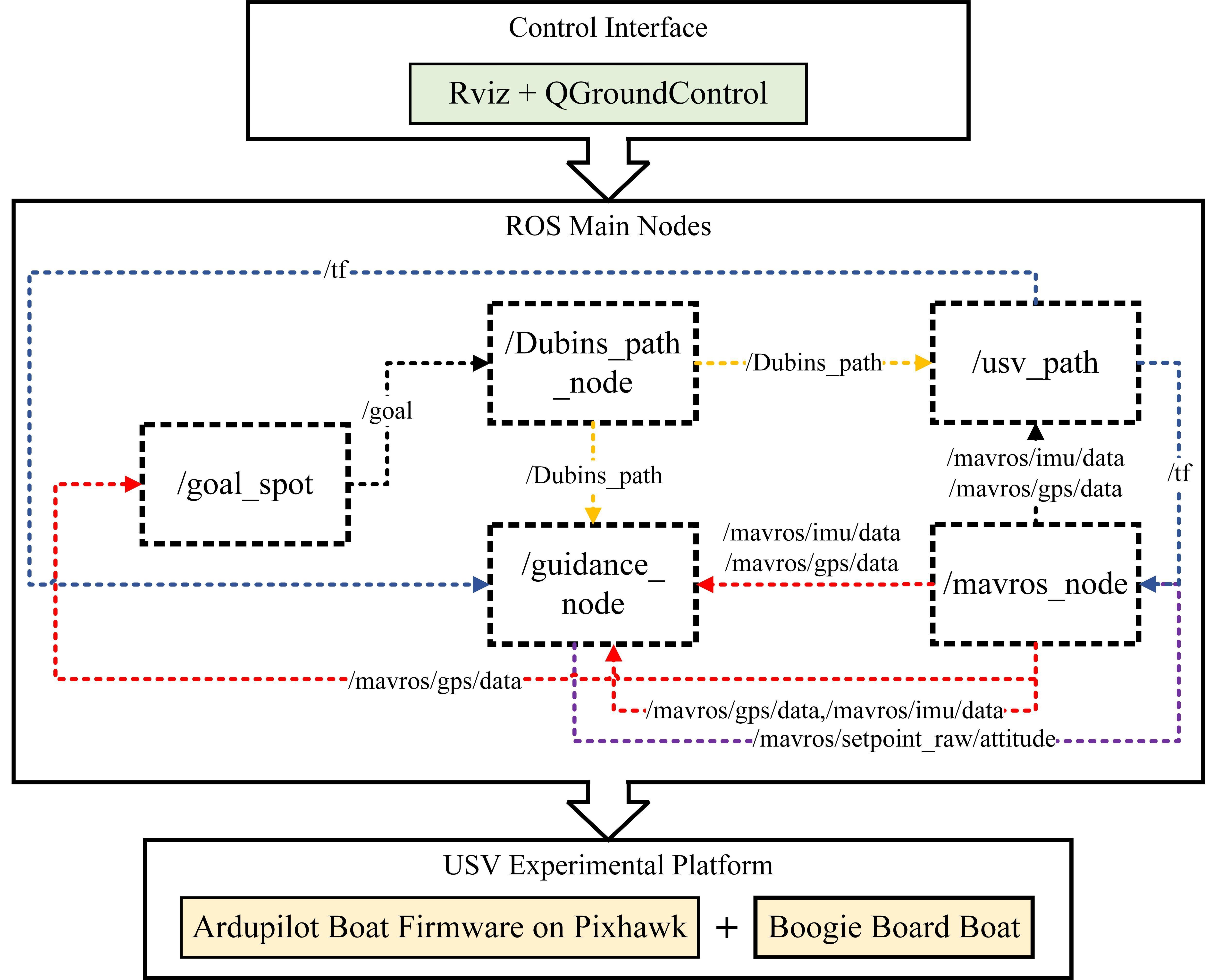}
    \caption{The implementation of the USV control system.}
    \label{ros_node}
\end{figure}

\section{Results and discussions}\label{sec4}

To verify that the VFALOS method is better, its performance is compared to that of the vector field-based integral LOS guidance law (abbreviated as VFILOS) in the reference \cite{wang2023vector} and the time-varying forward-looking distance LOS guidance law (abbreviated as TLOS) in the reference \cite{lenes2019autonomous}.  The USV control nodes are displayed in Fig. \ref{ros_node}. Fig. \ref{LOSfangzhen} illustrates that VFALOS performs better when it comes to tracing.

According to the above experimental comparisons, the VFALOS control law as a path tracking algorithm has better anti-interference ability and tracking accuracy.

\begin{figure*}
\centering
    \subfloat[TLOS\label{TLOS_path1}]
    {
     \includegraphics[width=5.5cm]{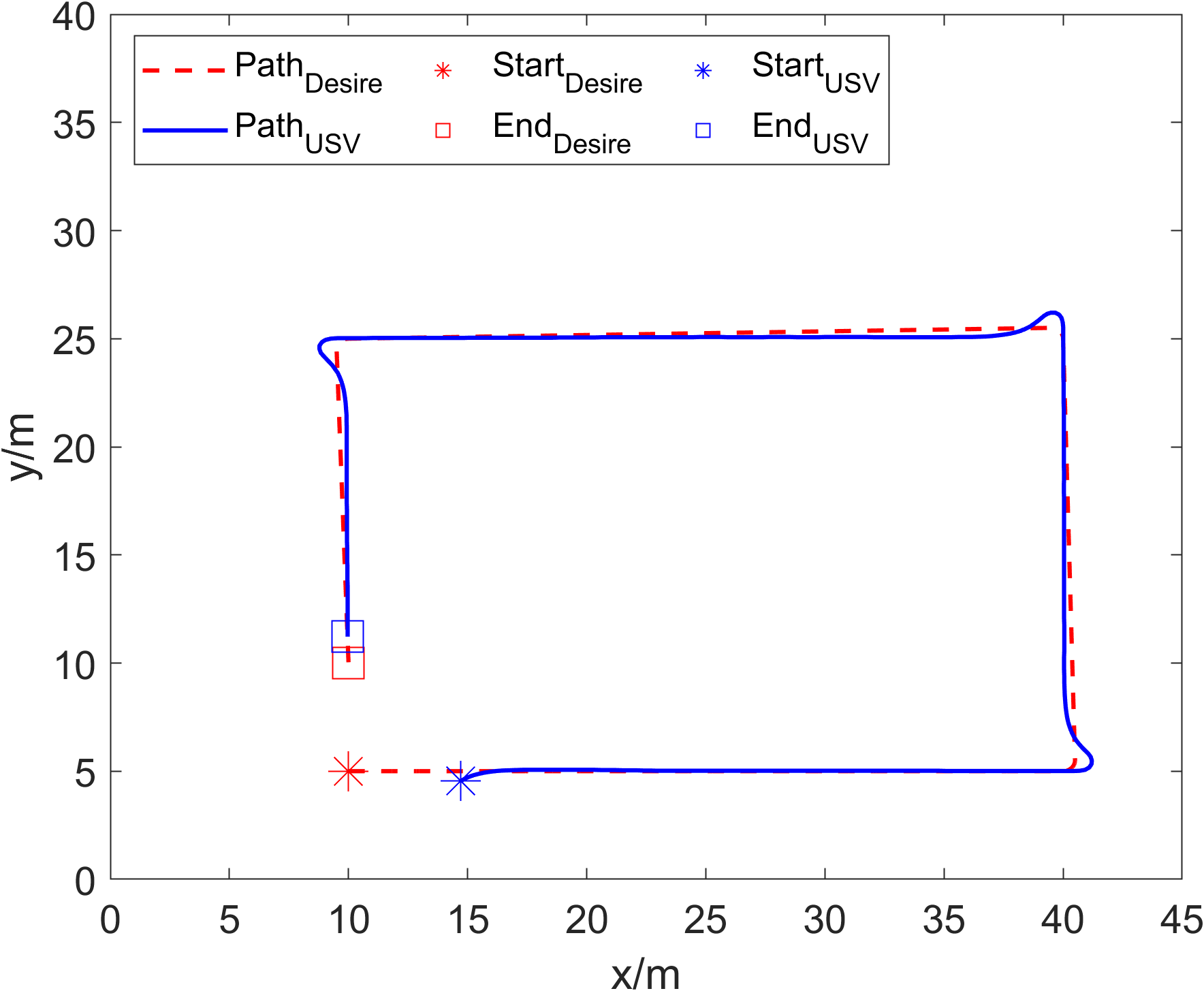}
    }
    \subfloat[VFILOS\centering\label{VFILOS_path1}]
    {
     \includegraphics[width=5.5cm]{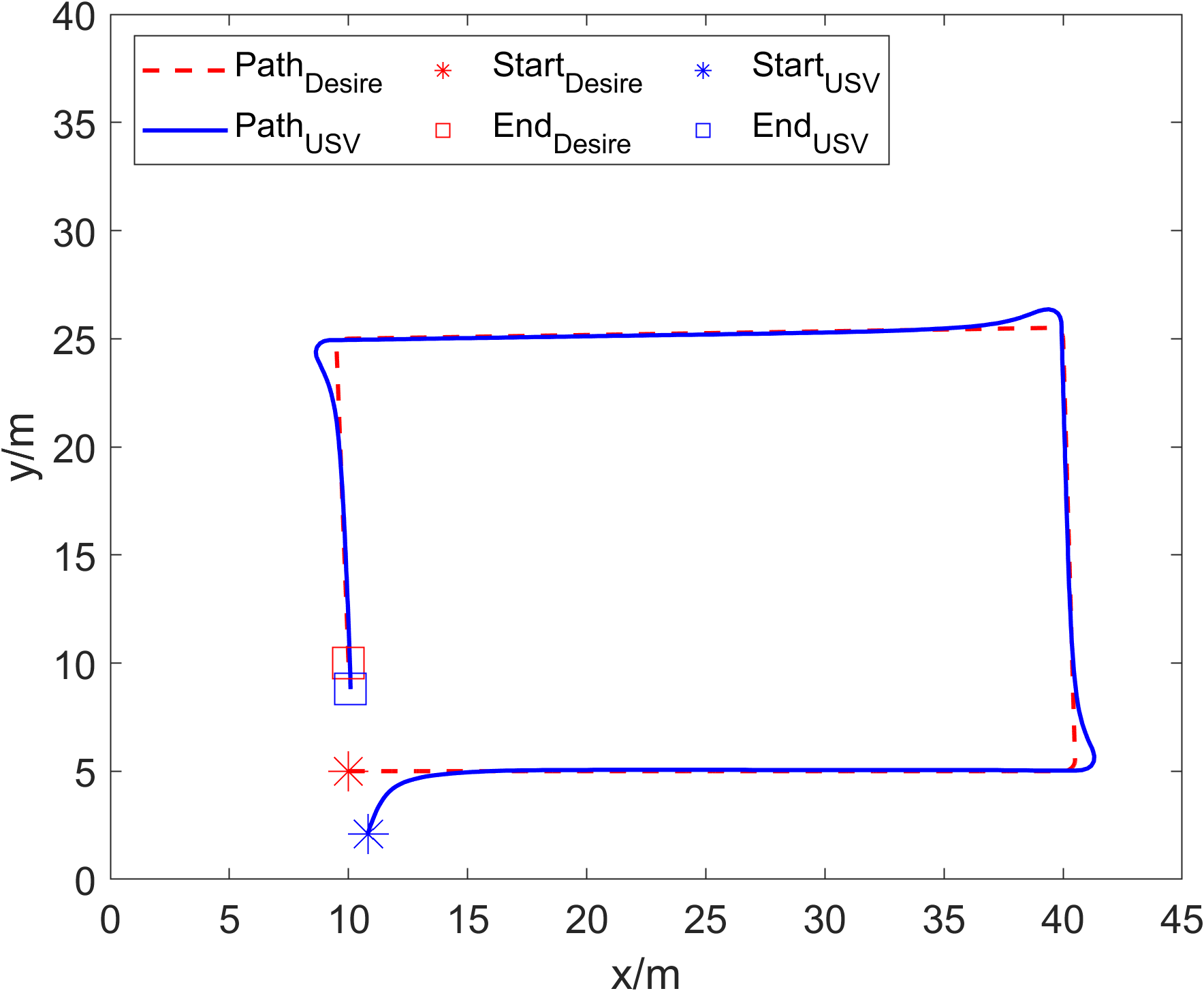}
    }
    \subfloat[VFALOS\centering\label{VFALOS_path1}]
    {
     \includegraphics[width=5.5cm]{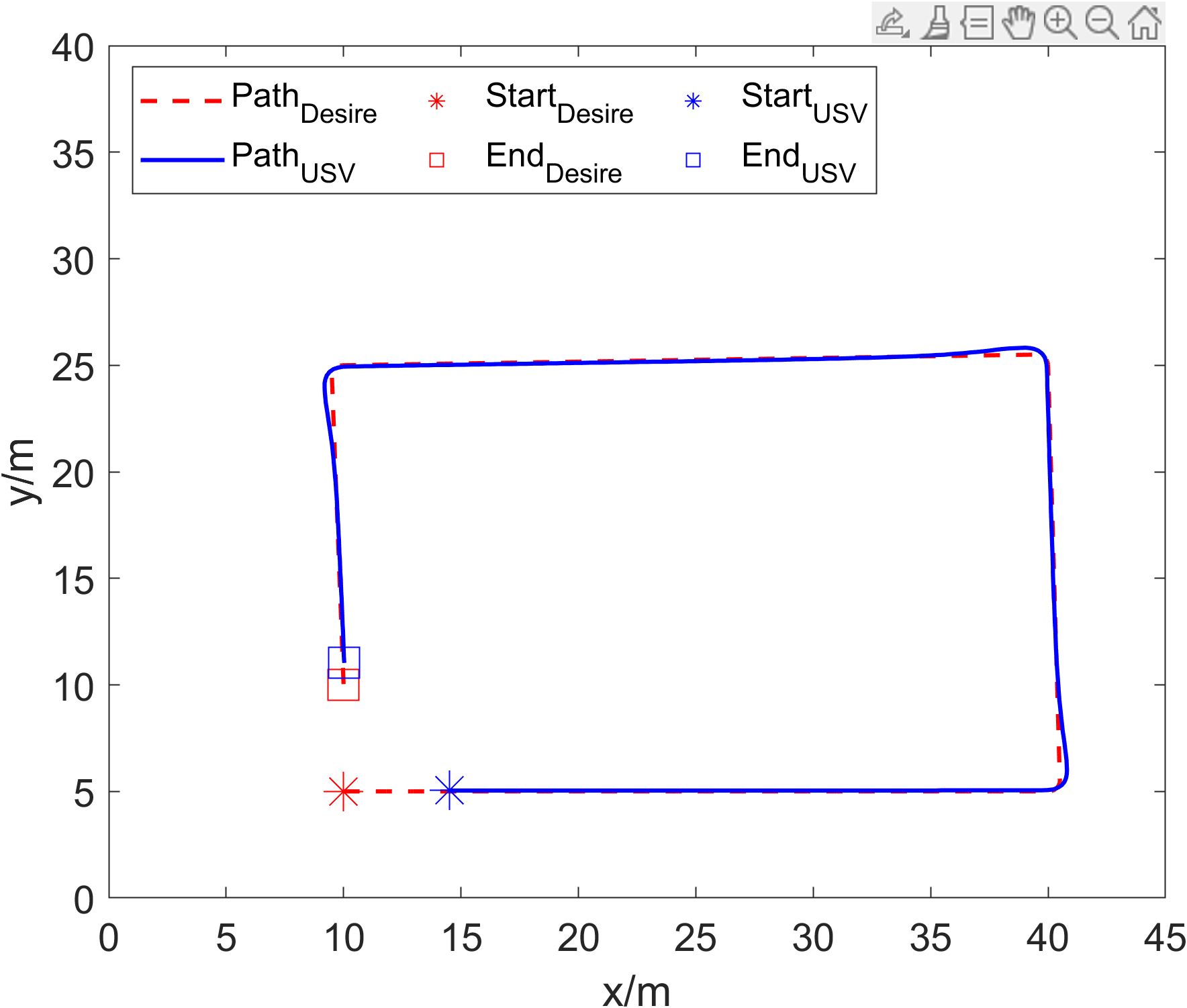}
    }
    \quad
    \subfloat[TLOS\label{TLOS_path2}]
    {
     \includegraphics[width=5.5cm]{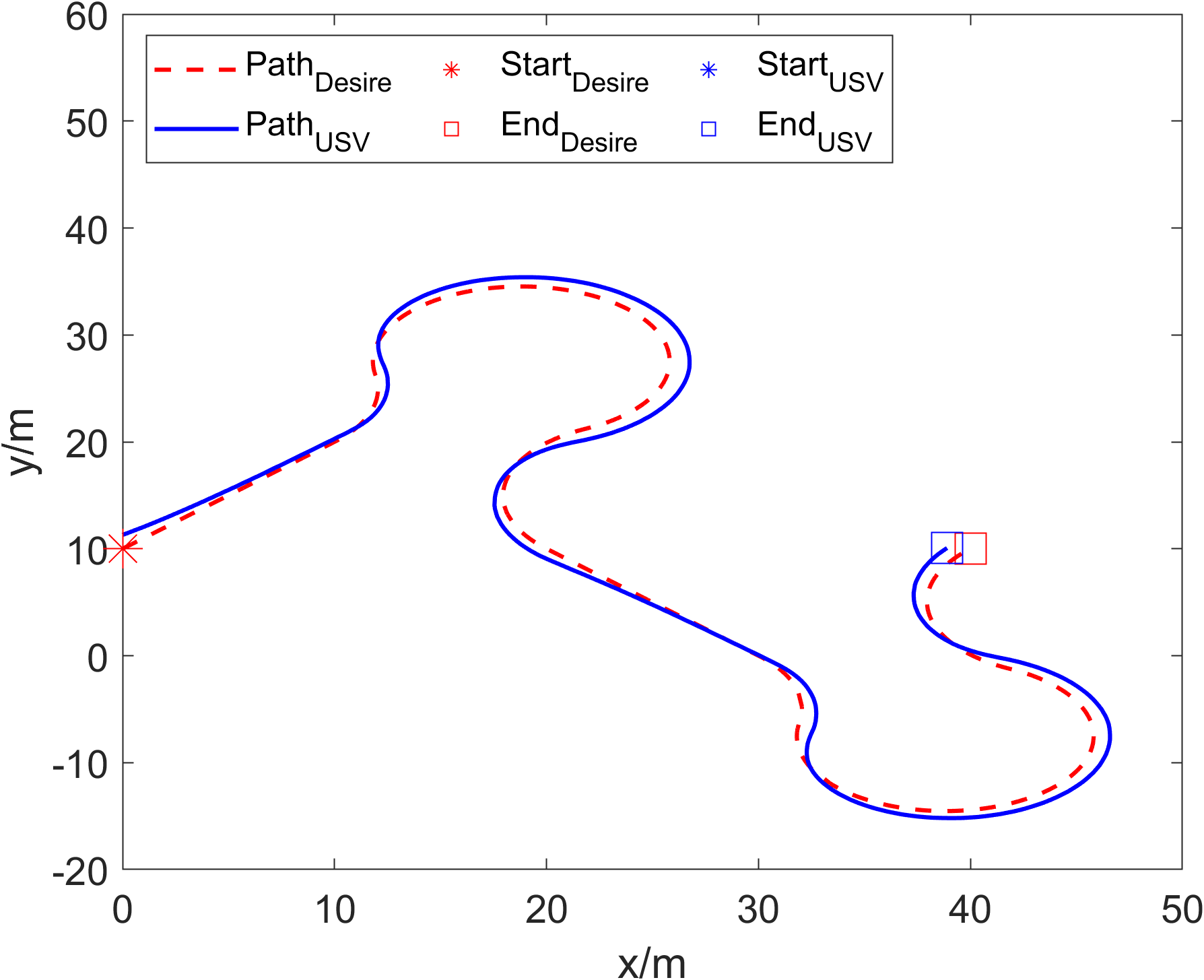}
    }
    \subfloat[VFILOS\centering\label{VFILOS_path2}]
    {
     \includegraphics[width=5.5cm]{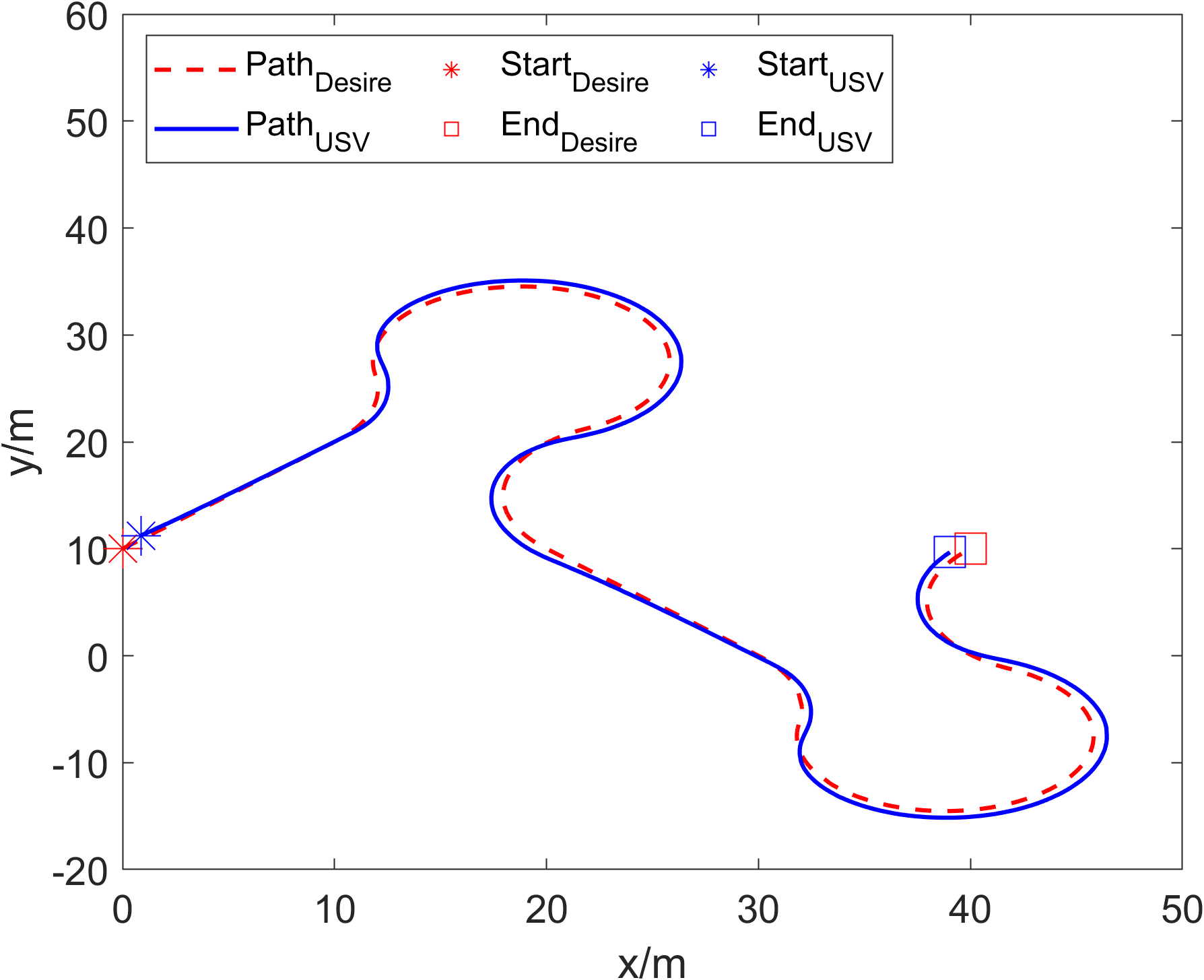}
    }
    \subfloat[VFALOS\centering\label{VFALOS_path2}]
    {
     \includegraphics[width=5.5cm]{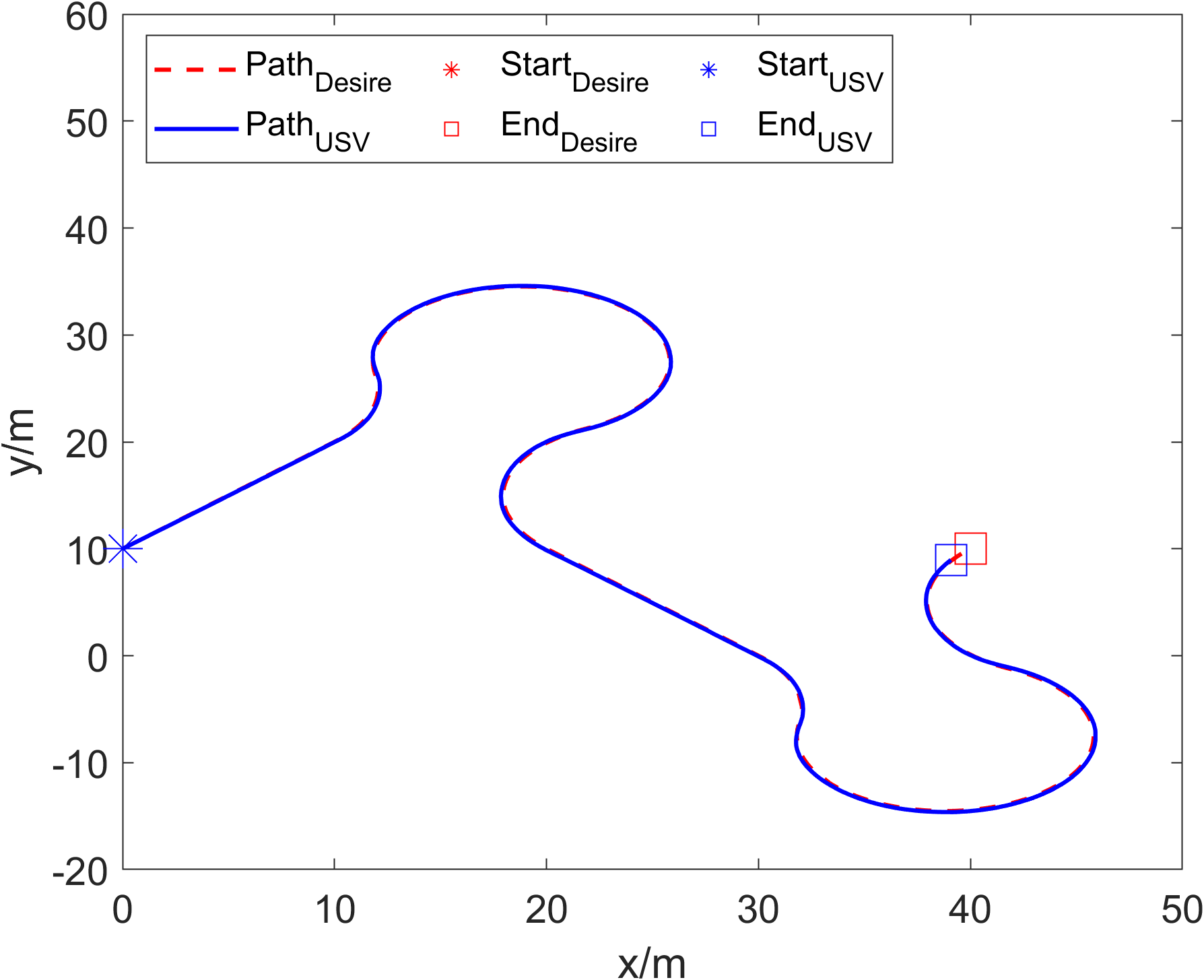}
    }
    \quad
    \subfloat[TLOS\label{TLOS_path3}]
    {
     \includegraphics[width=5.5cm]{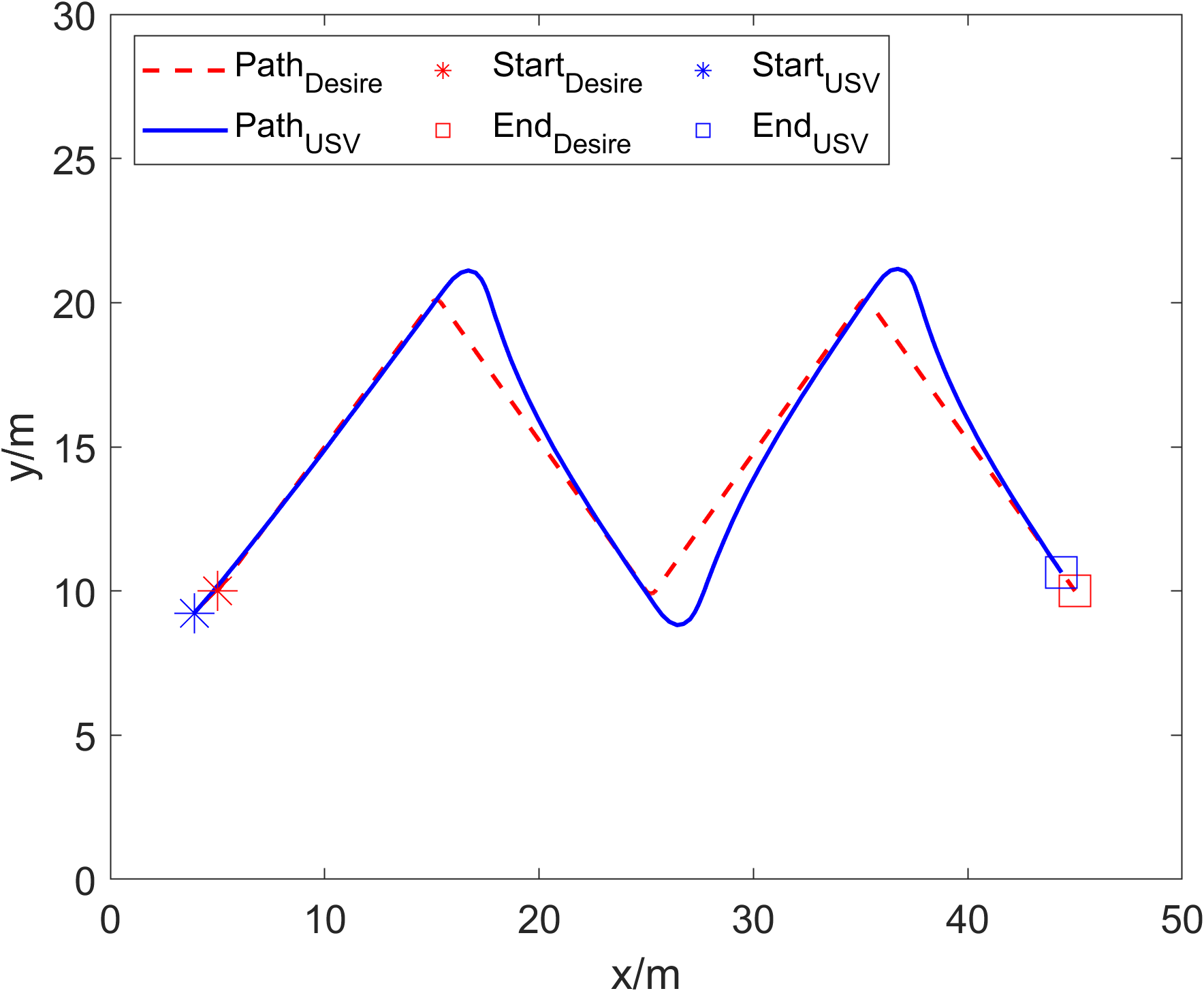}
    }
    \subfloat[VFILOS\centering\label{VFILOS_path3}]
    {
     \includegraphics[width=5.5cm]{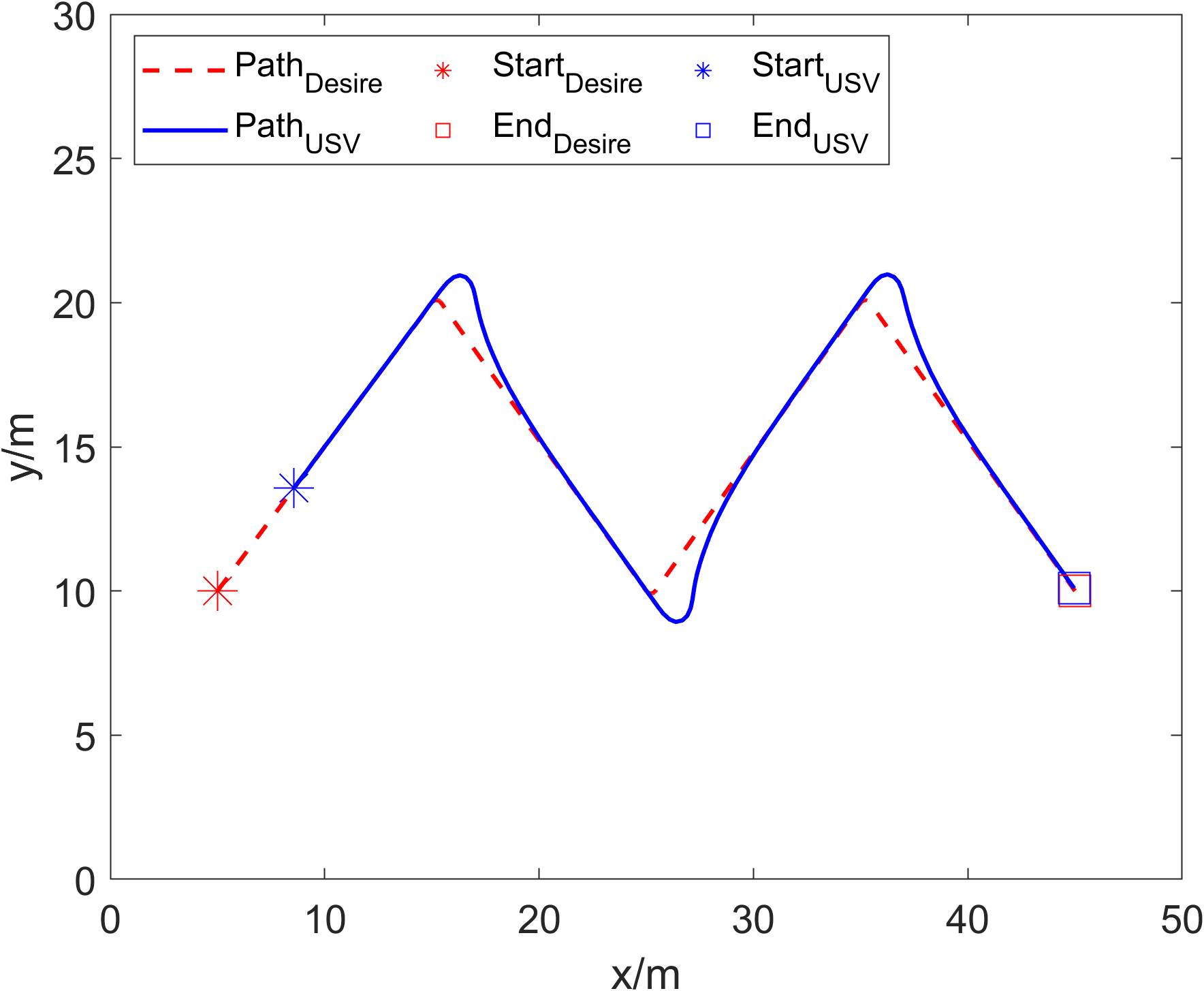}
    }
    \subfloat[VFALOS\centering\label{VFALOS_path3}]
    {
     \includegraphics[width=5.5cm]{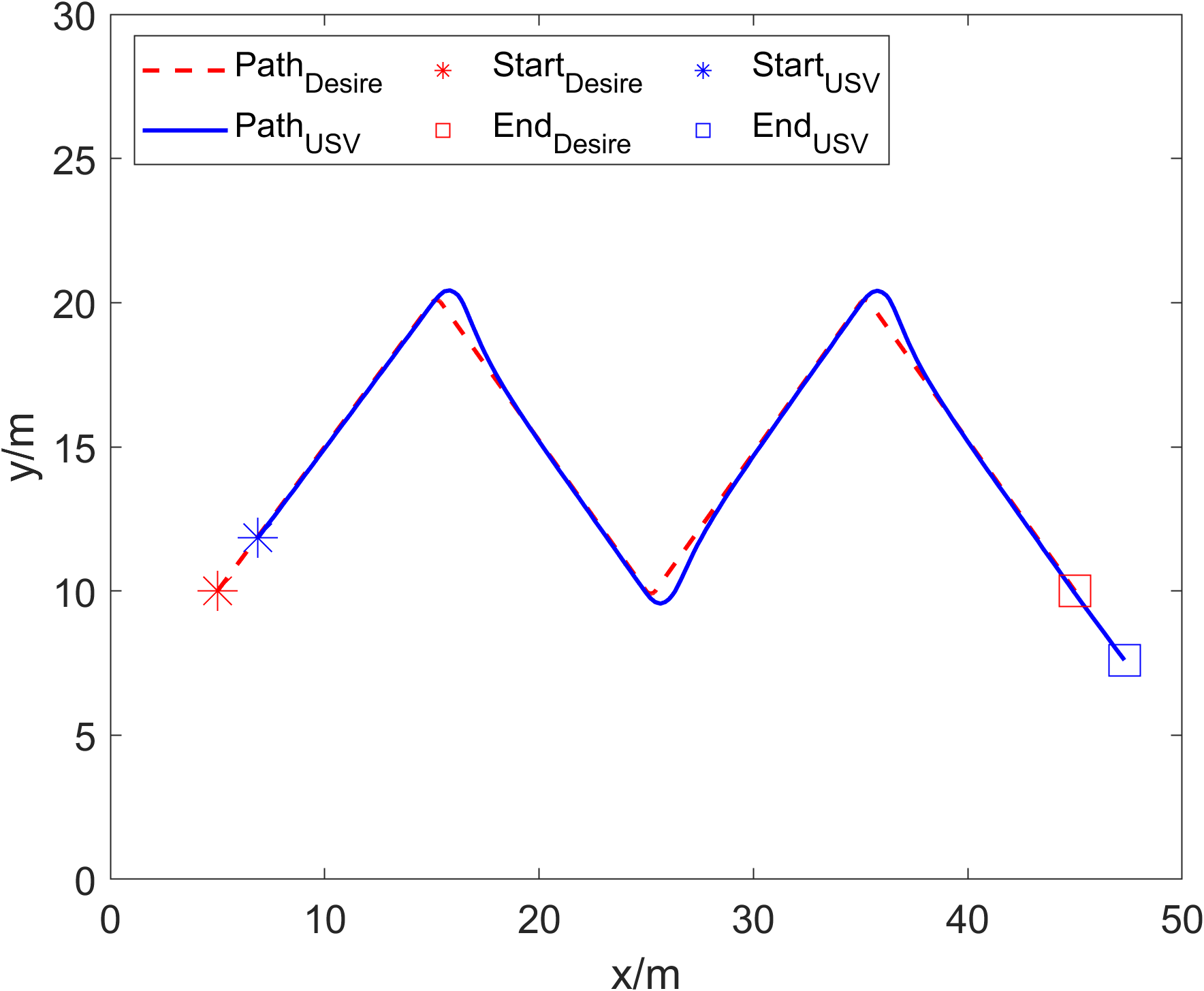}
    }
   \caption{USV path tracking renderings using three different LOS guidance laws in  the simulation experiment.}\label{LOSfangzhen}
\end{figure*}

\begin{figure*}
\centering
\includegraphics[width=12cm]{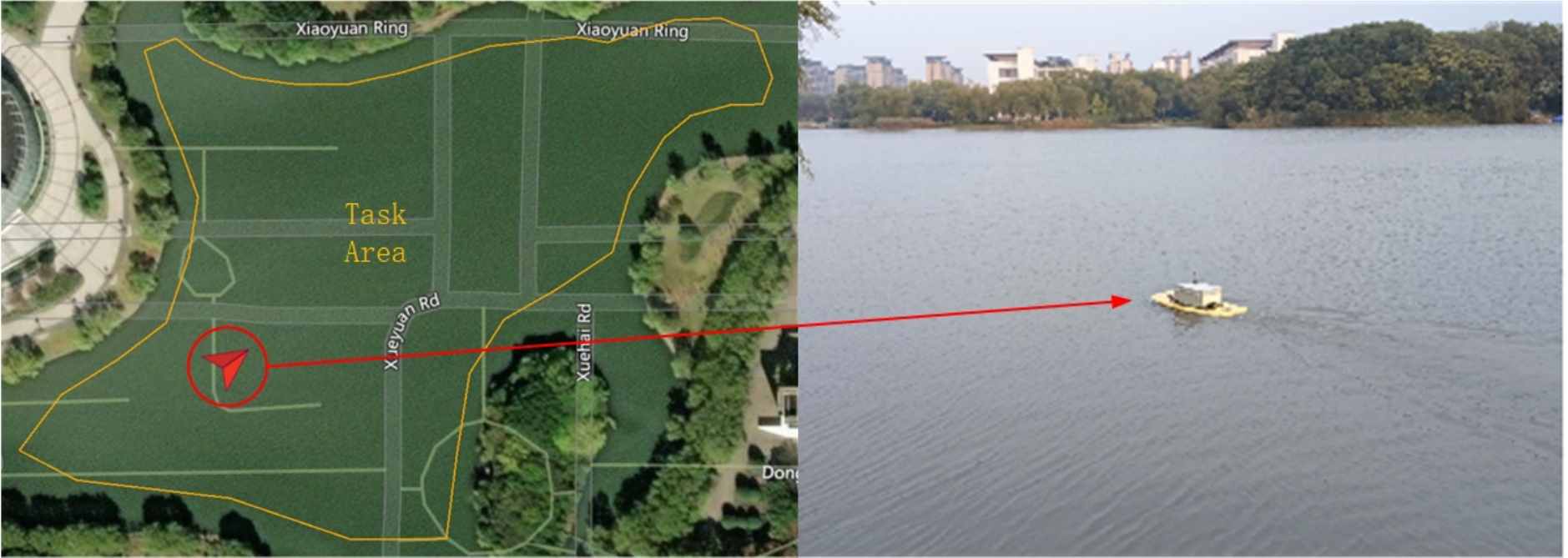}
\caption{Experimental site and USV experimental platform.}
\label{fig11}
\end{figure*}

\begin{figure*}
\centering
    \subfloat[TLOS\label{fig2a}]
    {
     \includegraphics[width=5.5cm]{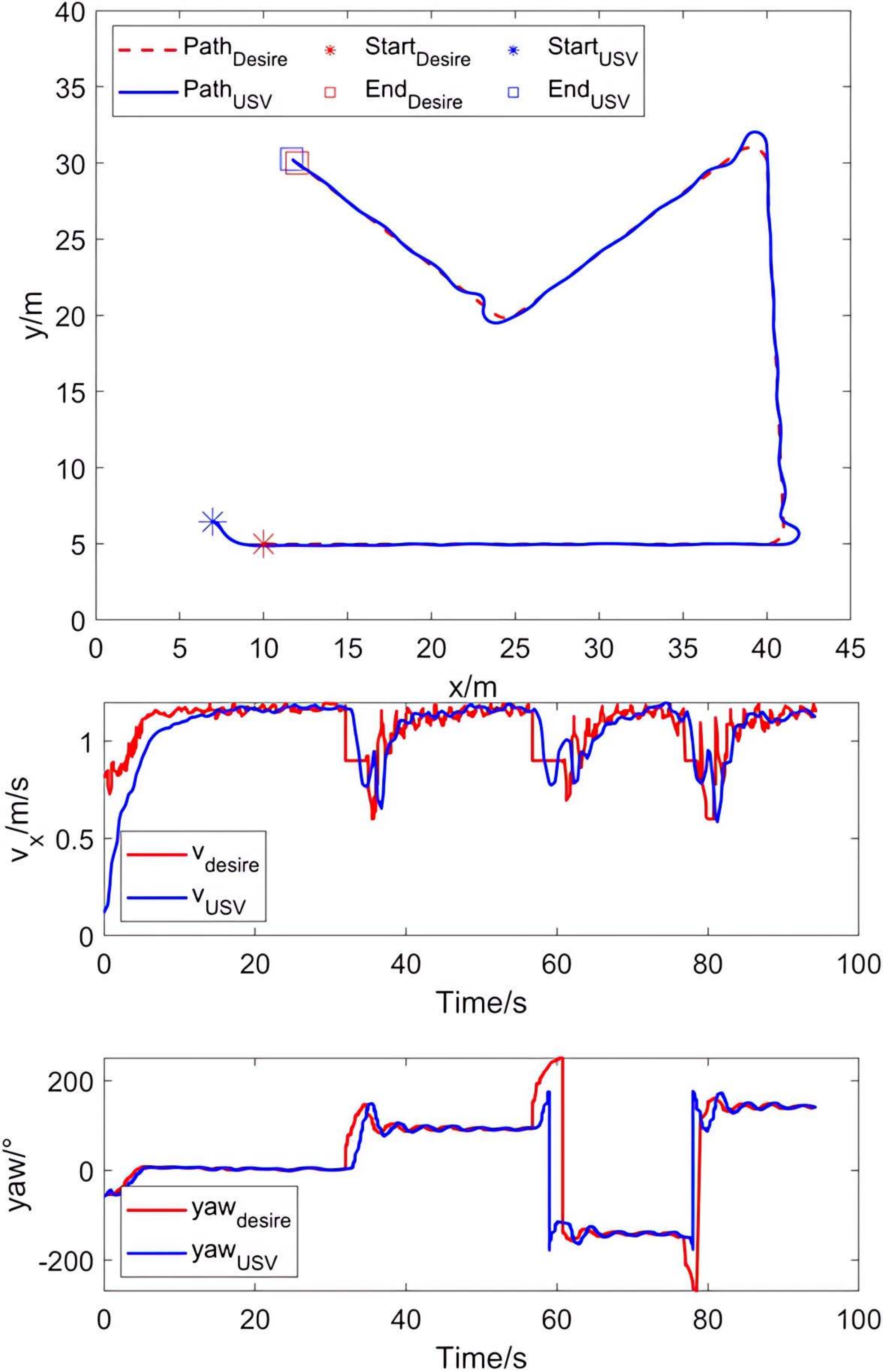}
    }
    \subfloat[VFILOS\centering\label{fig2b}]
    {
     \includegraphics[width=5.5cm]{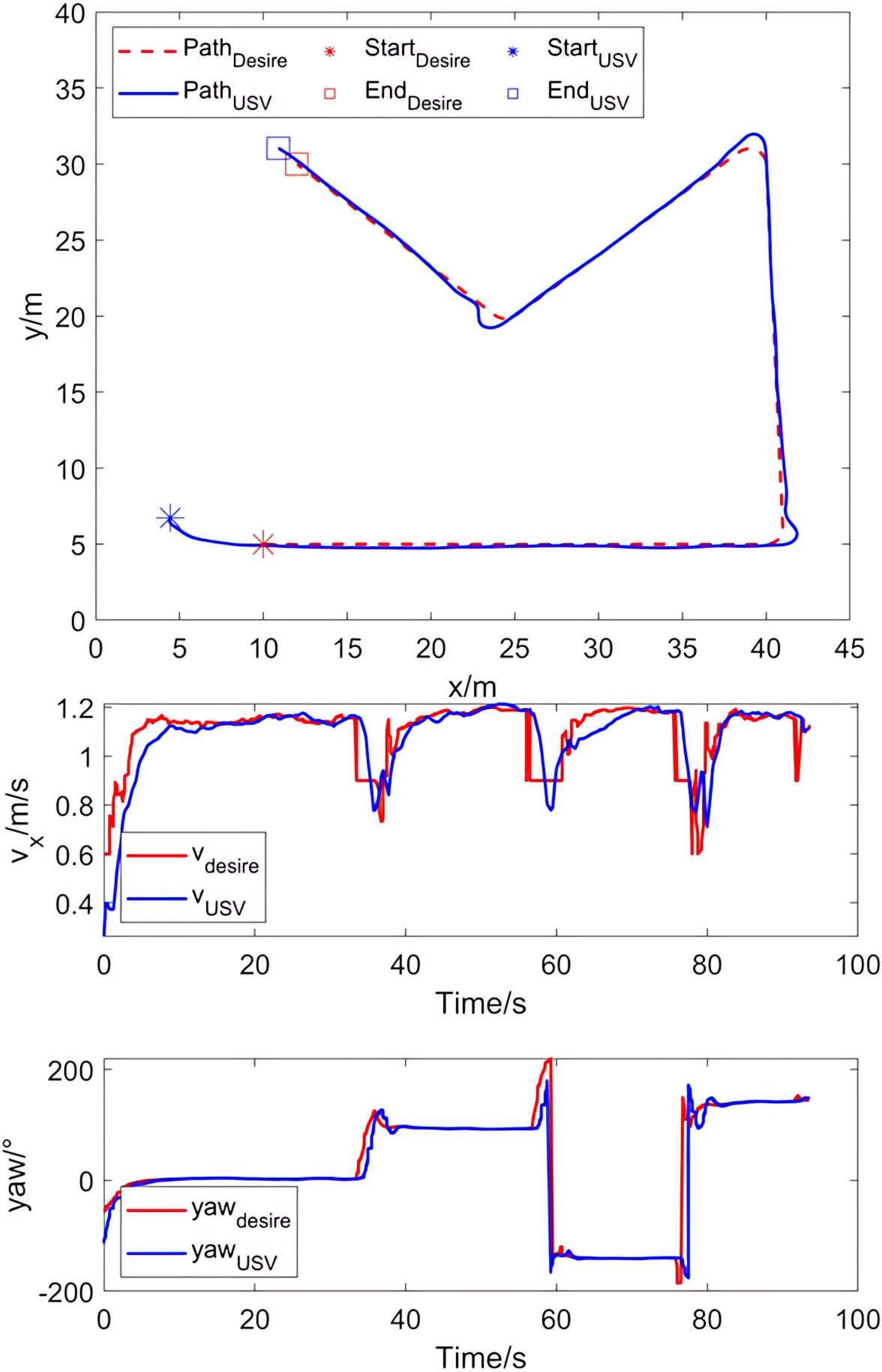}
    }
    \subfloat[VFALOS\centering\label{fig2c}]
    {
     \includegraphics[width=5.5cm]{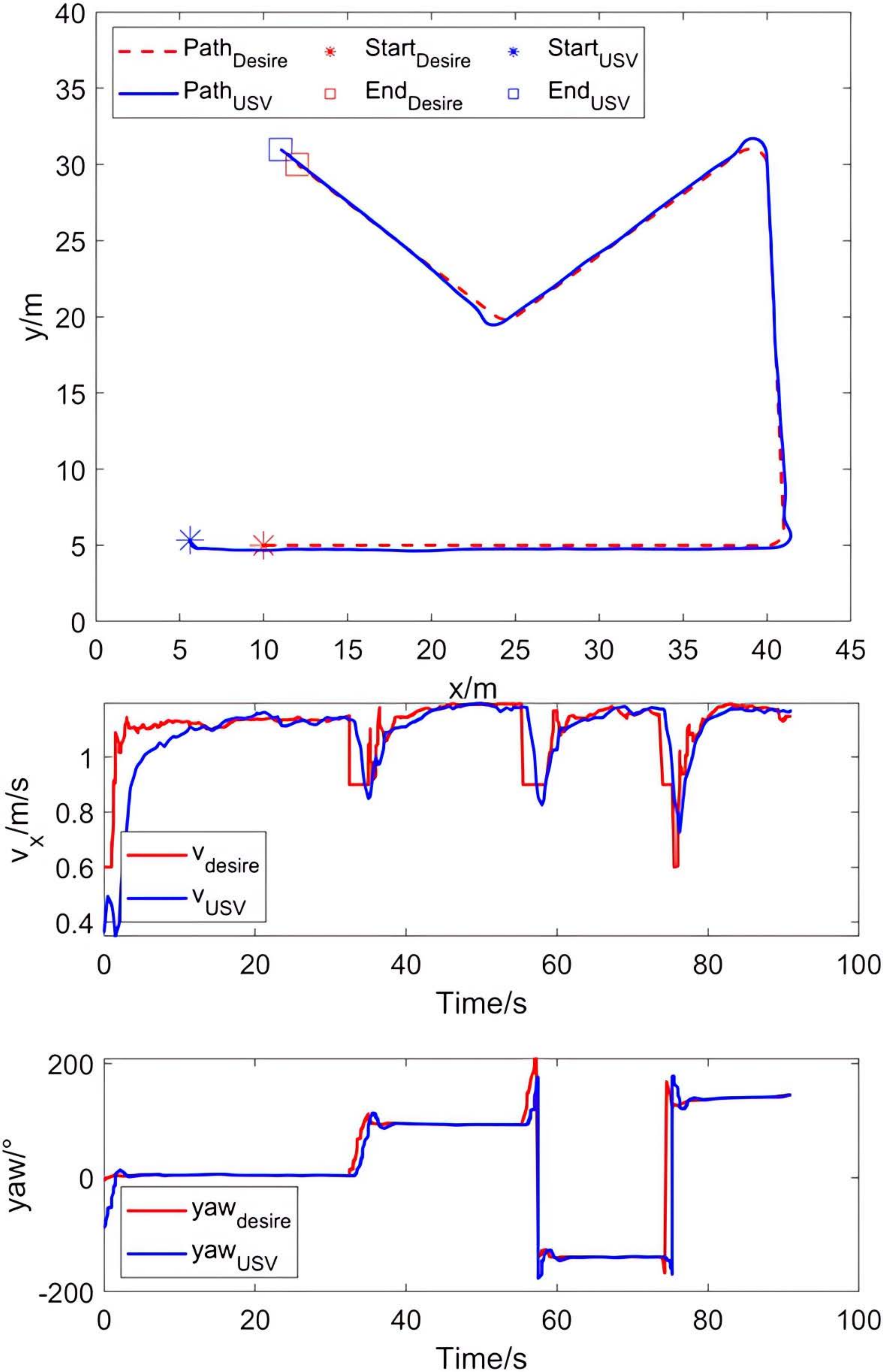}
    }
   \caption{USV path tracking renderings using three different LOS guidance laws in the lake experiment.}\label{fig2}
\end{figure*}

\section{Conclusion}\label{sec5}
In this paper, a guidance law based on the concept of USV vector field is proposed for tracking desired paths. Using $\kappa$-exponential stability, it is demonstrated that the proposed guidance law provides asymptotic tracking on both straight and curved paths. Simulation experiments and lake experiments using ROS verified the effectiveness of the guidance method. The path tracking error is smaller than similar guidance laws for both straight and curved paths. The guidance law also enables the USV to track a given route stably. In the future, we will consider improving the adaptive guidance control of underdriven USVs under strong interference conditions.



\section{Acknowledgements}



\bibliography{refs}

\begin{thebibliography}{18}
\expandafter\ifx\csname natexlab\endcsname\relax\def\natexlab#1{#1}\fi
\providecommand{\url}[1]{\texttt{#1}}
\providecommand{\href}[2]{#2}
\providecommand{\path}[1]{#1}
\providecommand{\DOIprefix}{doi:}
\providecommand{\ArXivprefix}{arXiv:}
\providecommand{\URLprefix}{URL: }
\providecommand{\Pubmedprefix}{pmid:}
\providecommand{\doi}[1]{\href{http://dx.doi.org/#1}{\path{#1}}}
\providecommand{\Pubmed}[1]{\href{pmid:#1}{\path{#1}}}
\providecommand{\bibinfo}[2]{#2}
\ifx\xfnm\relax \def\xfnm[#1]{\unskip,\space#1}\fi
\bibitem[{Barrera et~al.(2021)Barrera, Padron, Luis and Llinas}]{barrera2021trends}
\bibinfo{author}{Barrera, C.}, \bibinfo{author}{Padron, I.}, \bibinfo{author}{Luis, F.}, \bibinfo{author}{Llinas, O.}, \bibinfo{year}{2021}.
\newblock \bibinfo{title}{Trends and challenges in unmanned surface vehicles (usv): From survey to shipping}.
\newblock \bibinfo{journal}{TransNav: International Journal on Marine Navigation and Safety of Sea Transportation} \bibinfo{volume}{15}, \bibinfo{pages}{135--142}.
\bibitem[{Borhaug et~al.(2008)Borhaug, Pavlov and Pettersen}]{Borhaug2008}
\bibinfo{author}{Borhaug, E.}, \bibinfo{author}{Pavlov, A.}, \bibinfo{author}{Pettersen, K.Y.}, \bibinfo{year}{2008}.
\newblock \bibinfo{title}{Integral los control for path following of underactuated marine surface vessels in the presence of constant ocean currents}, in: \bibinfo{booktitle}{2008 47th IEEE conference on decision and control}, \bibinfo{organization}{IEEE}. pp. \bibinfo{pages}{4984--4991}.
\bibitem[{Breivik and Fossen(2009)}]{Breivik2009}
\bibinfo{author}{Breivik, M.}, \bibinfo{author}{Fossen, T.I.}, \bibinfo{year}{2009}.
\newblock \bibinfo{title}{Guidance laws for autonomous underwater vehicles}.
\newblock \bibinfo{journal}{Intelligent Underwater Vehicles} \bibinfo{volume}{4}, \bibinfo{pages}{51--76}.
\bibitem[{Caharija et~al.(2016)Caharija, Pettersen, Bibuli, Calado, Zereik, Braga, Gravdahl, S{\o}rensen, Milovanovi{\'c} and Bruzzone}]{Caharija2016}
\bibinfo{author}{Caharija, W.}, \bibinfo{author}{Pettersen, K.Y.}, \bibinfo{author}{Bibuli, M.}, \bibinfo{author}{Calado, P.}, \bibinfo{author}{Zereik, E.}, \bibinfo{author}{Braga, J.}, \bibinfo{author}{Gravdahl, J.T.}, \bibinfo{author}{S{\o}rensen, A.J.}, \bibinfo{author}{Milovanovi{\'c}, M.}, \bibinfo{author}{Bruzzone, G.}, \bibinfo{year}{2016}.
\newblock \bibinfo{title}{Integral line-of-sight guidance and control of underactuated marine vehicles: Theory, simulations, and experiments}.
\newblock \bibinfo{journal}{IEEE Transactions on Control Systems Technology} \bibinfo{volume}{24}, \bibinfo{pages}{1623--1642}.
\bibitem[{Du et~al.(2023)Du, Yang, Chen and Huang}]{du2023improved}
\bibinfo{author}{Du, P.}, \bibinfo{author}{Yang, W.}, \bibinfo{author}{Chen, Y.}, \bibinfo{author}{Huang, S.}, \bibinfo{year}{2023}.
\newblock \bibinfo{title}{Improved indirect adaptive line-of-sight guidance law for path following of under-actuated auv subject to big ocean currents}.
\newblock \bibinfo{journal}{Ocean Engineering} \bibinfo{volume}{281}, \bibinfo{pages}{114729}.
\bibitem[{Fossen(2023)}]{fossen2023adaptive}
\bibinfo{author}{Fossen, T.I.}, \bibinfo{year}{2023}.
\newblock \bibinfo{title}{An adaptive line-of-sight (alos) guidance law for path following of aircraft and marine craft}.
\newblock \bibinfo{journal}{IEEE Transactions on Control Systems Technology} \bibinfo{volume}{31}, \bibinfo{pages}{2887--2894}.
\bibitem[{Lefeber(2000)}]{lefeber2000tracking}
\bibinfo{author}{Lefeber, A.A.J.}, \bibinfo{year}{2000}.
\newblock \bibinfo{title}{Tracking control of nonlinear mechanical systems}.
\newblock \bibinfo{publisher}{The Netherlands:Universiteit Twente}.
\bibitem[{Lenes(2019)}]{lenes2019autonomous}
\bibinfo{author}{Lenes, J.H.}, \bibinfo{year}{2019}.
\newblock \bibinfo{title}{Autonomous online path planning and path-following control for complete coverage maneuvering of a USV}.
\newblock Master's thesis. NTNU.
\bibitem[{Liu et~al.(2016)Liu, Wang, Peng and Wang}]{liu2016predictor}
\bibinfo{author}{Liu, L.}, \bibinfo{author}{Wang, D.}, \bibinfo{author}{Peng, Z.}, \bibinfo{author}{Wang, H.}, \bibinfo{year}{2016}.
\newblock \bibinfo{title}{Predictor-based los guidance law for path following of underactuated marine surface vehicles with sideslip compensation}.
\newblock \bibinfo{journal}{Ocean Engineering} \bibinfo{volume}{124}, \bibinfo{pages}{340--348}.
\bibitem[{Liu et~al.(2020)Liu, Zhang and Wang}]{liu2020adaptive}
\bibinfo{author}{Liu, X.}, \bibinfo{author}{Zhang, M.}, \bibinfo{author}{Wang, S.}, \bibinfo{year}{2020}.
\newblock \bibinfo{title}{Adaptive region tracking control with prescribed transient performance for autonomous underwater vehicle with thruster fault}.
\newblock \bibinfo{journal}{Ocean Engineering} \bibinfo{volume}{196}, \bibinfo{pages}{106804}.
\bibitem[{Mccue(2016)}]{mccue2016handbook}
\bibinfo{author}{Mccue, L.}, \bibinfo{year}{2016}.
\newblock \bibinfo{title}{Handbook of marine craft hydrodynamics and motion control [bookshelf]}.
\newblock \bibinfo{journal}{IEEE Control Systems Magazine} \bibinfo{volume}{36}, \bibinfo{pages}{78--79}.
\bibitem[{Nelson et~al.(2007a)Nelson, Barber, McLain and Beard}]{Nelson2007}
\bibinfo{author}{Nelson, D.R.}, \bibinfo{author}{Barber, D.B.}, \bibinfo{author}{McLain, T.W.}, \bibinfo{author}{Beard, R.W.}, \bibinfo{year}{2007}a.
\newblock \bibinfo{title}{Vector field path following for miniature air vehicles}.
\newblock \bibinfo{journal}{IEEE Transactions on Robotics} \bibinfo{volume}{23}, \bibinfo{pages}{519--529}.
\bibitem[{Nelson et~al.(2007b)Nelson, Barber, McLain and Beard}]{nelson2007vector}
\bibinfo{author}{Nelson, D.R.}, \bibinfo{author}{Barber, D.B.}, \bibinfo{author}{McLain, T.W.}, \bibinfo{author}{Beard, R.W.}, \bibinfo{year}{2007}b.
\newblock \bibinfo{title}{Vector field path following for miniature air vehicles}.
\newblock \bibinfo{journal}{IEEE Transactions on Robotics} \bibinfo{volume}{23}, \bibinfo{pages}{519--529}.
\bibitem[{Peng et~al.(2021)Peng, Wang, Wang and Han}]{pengusv2021}
\bibinfo{author}{Peng, Z.}, \bibinfo{author}{Wang, J.}, \bibinfo{author}{Wang, D.}, \bibinfo{author}{Han, Q.L.}, \bibinfo{year}{2021}.
\newblock \bibinfo{title}{An overview of recent advances in coordinated control of multiple autonomous surface vehicles}.
\newblock \bibinfo{journal}{IEEE Transactions on Industrial Informatics} \bibinfo{volume}{17}, \bibinfo{pages}{732--745}.
\bibitem[{Qiu et~al.(2020)Qiu, Wang and Fan}]{qiu2020predictor}
\bibinfo{author}{Qiu, B.}, \bibinfo{author}{Wang, G.}, \bibinfo{author}{Fan, Y.}, \bibinfo{year}{2020}.
\newblock \bibinfo{title}{Predictor los-based trajectory linearization control for path following of underactuated unmanned surface vehicle with input saturation}.
\newblock \bibinfo{journal}{Ocean Engineering} \bibinfo{volume}{214}, \bibinfo{pages}{107874}.
\bibitem[{Wan et~al.(2020)Wan, Su, Zhang, Shi and AbouOmar}]{wan2020improved}
\bibinfo{author}{Wan, L.}, \bibinfo{author}{Su, Y.}, \bibinfo{author}{Zhang, H.}, \bibinfo{author}{Shi, B.}, \bibinfo{author}{AbouOmar, M.S.}, \bibinfo{year}{2020}.
\newblock \bibinfo{title}{An improved integral light-of-sight guidance law for path following of unmanned surface vehicles}.
\newblock \bibinfo{journal}{Ocean engineering} \bibinfo{volume}{205}, \bibinfo{pages}{107302}.
\bibitem[{Wang et~al.(2023)Wang, Su, Wu, Fan, Qi, Wang and Feng}]{wang2023vector}
\bibinfo{author}{Wang, M.}, \bibinfo{author}{Su, Y.}, \bibinfo{author}{Wu, N.}, \bibinfo{author}{Fan, Y.}, \bibinfo{author}{Qi, J.}, \bibinfo{author}{Wang, Y.}, \bibinfo{author}{Feng, Z.}, \bibinfo{year}{2023}.
\newblock \bibinfo{title}{Vector field-based integral los path following and target tracking for underactuated unmanned surface vehicle}.
\newblock \bibinfo{journal}{Ocean Engineering} \bibinfo{volume}{285}, \bibinfo{pages}{115462}.
\bibitem[{Yuan and Feng(2022)}]{yuan2022study}
\bibinfo{author}{Yuan, X.}, \bibinfo{author}{Feng, Z.}, \bibinfo{year}{2022}.
\newblock \bibinfo{title}{A study of path planning for multi-uavs in random obstacle environment based on improved artificial potential field method}, in: \bibinfo{booktitle}{2022 China Automation Congress (CAC)}, \bibinfo{publisher}{IEEE}. pp. \bibinfo{pages}{5241--5245}.

\end{thebibliography}
\bibliographystyle{cas-model2-names}

\bio{}
\endbio

\bio{}
\endbio

\end{document}